\documentclass[11pt]{article}
\pdfoutput=1

\usepackage{fullpage,times}

\usepackage[utf8]{inputenc}
\usepackage[T1]{fontenc}
\usepackage{hyperref}
\hypersetup{
        colorlinks   = true,
        urlcolor     = blue,
        linkcolor    = blue,
        citecolor   = blue
}

\usepackage{url}
\usepackage{booktabs}
\usepackage{amsmath,amsfonts,amsthm,amssymb}
\usepackage{nicefrac}
\usepackage{microtype}
\usepackage{natbib}

\usepackage{enumitem}
\usepackage{algorithm}
\usepackage{algorithmicx}
\usepackage{algpseudocode}
\usepackage[nolist,nohyperlinks]{acronym}
\usepackage{color}
\usepackage{bold-extra}
\usepackage{mathtools}
\usepackage{etoolbox}
\usepackage[capitalize]{cleveref}
\usepackage{authblk}

\renewcommand{\P}{\mathbb{P}}

\newcommand{\sA}{\mathcal{A}}

\newcommand{\sM}{\mathcal{M}}

\newcommand{\sD}{\mathcal{D}}
\newcommand{\sF}{\mathcal{F}}

\newcommand{\sZ}{\mathcal{Z}}

\newcommand{\ind}[1]{\mathbb{I}\left\{ #1 \right\}}

\newcommand{\field}[1]{\mathbb{#1}}

\newcommand{\R}{\field{R}}

\newcommand{\Var}{\mathrm{Var}}

\newcommand{\scO}{\mathcal{O}}

\newcommand{\wh}{\hat} 
\newcommand{\ve}{\varepsilon}

\newcommand{\spin}{\{-1,+1\}}

\newtheorem{lemma}{Lemma}

\newtheorem{theorem}{Theorem}
\newtheorem{cor}{Corollary}
\newtheorem{prop}{Proposition}

\newcommand{\reals}{\mathbb{R}}

\newcommand{\Srepk}{{S^{(k)}}}

\newcommand{\repj}{^{(j)}}
\newcommand{\repk}{^{(k)}}
\newcommand{\repkj}{^{(k,j)}}

\DeclareMathOperator{\E}{\mathbb{E}}

\newcommand*\diff{\mathop{}\!\mathrm{d}}

\newcommand{\KL}{\mathrm{KL}}

\newcommand{\pr}[1]{\left( #1 \right)}
\newcommand{\br}[1]{\left[ #1 \right]}
\newcommand{\cbr}[1]{\left\{ #1 \right\}}
\newcommand{\abs}[1]{\left|#1\right|}

\newcommand{\wt}{\widetilde}

\newcommand{\Vtilrep}{V}

\newcommand{\delk}{^{\backslash k}}

\newcommand{\fwa}{f^{\text{\scshape{wa}}}}
\newcommand{\Vwa}{V^{\text{\scshape{wa}}}}

\newcommand{\WIS}{^{\text{\scshape{wis}}}}

\newcommand{\conf}{x}
\newcommand{\param}{{\theta}}

\newcommand{\ps}{\wh{p}_S}

\newcommand{\leqC}{\stackrel{\scO}{=}}
\newcommand{\leqCln}{\stackrel{\wt{\scO}}{=}}
\newcommand{\VES}{V^{\text{\scshape{es}}}}

\begin{acronym}
\acro{RLS}{Regularized Least Squares}
\acro{ERM}{Empirical Risk Minimization}
\acro{RKHS}{Reproducing kernel Hilbert space}
\acro{DA}{Domain Adaptation}
\acro{PSD}{Positive Semi-Definite}
\acro{SGD}{Stochastic Gradient Descent}
\acro{OGD}{Online Gradient Descent}
\acro{SGLD}{Stochastic Gradient Langevin Dynamics}
\acro{IS}{Importance Sampling}
\acro{WIS}{Weighted Importance Sampling}
\acro{MGF}{Moment-Generating Function}
\acro{ES}{Efron-Stein}
\acro{ESS}{Effective Sample Size}
\acro{KL}{Kullback-Liebler}
\end{acronym}

\title{Efron-Stein PAC-Bayesian Inequalities}

\author{Ilja Kuzborskij}
\author{Csaba Szepesv\'ari}
\affil{DeepMind}
\date{}

\begin{document}
\maketitle
\begin{abstract}
We prove semi-empirical concentration inequalities for random variables which are given as possibly nonlinear functions of independent random variables.
These inequalities describe concentration of random variable in terms of the data/distribution-dependent \ac{ES} estimate of its variance and they do not require any additional assumptions on the moments.
In particular, this allows us to state semi-empirical Bernstein type inequalities for general functions of \emph{unbounded} random variables, which gives user-friendly concentration bounds for cases where related methods (e.g.\ bounded differences) might be more challenging to apply.
We extend these results to \emph{Efron-Stein PAC-Bayesian inequalities} which hold for arbitrary probability kernels that define a random, data-dependent choice of the function of interest.
Finally, we demonstrate a number of applications, including PAC-Bayesian generalization bounds for unbounded loss functions, empirical Bernstein type generalization bounds, new \emph{truncation-free} bounds for \emph{off-policy} evaluation with \ac{WIS}, and off-policy PAC-Bayesian learning with \ac{WIS}.
\end{abstract}
\section{Introduction}
In the following we will be concerned with bounds on the upper tail probability of
\[
  \Delta = f(S) - \E[f(S)]~,
\]
where $S = \pr{X_1, X_2, \ldots, X_n}$ composed from independent random elements is distributed according to some probability measure $\sD \in \sM_1(\sZ)$\footnote{We use notation $\sM_1(\sA)$ to denote a family of probability measures supported on a set $\sA$.}, and $\sZ = \sZ_1 \times \dots \times \sZ_n$ is some space and  $f : \sZ \to \reals$ is either a fixed measurable function, or is a function that is randomly chosen as a function of $S$.

We first consider a simpler case of such \emph{concentration inequalities}, when $f$ is a fixed function and
the user can choose from a number of different ways to study behavior of $\Delta$ (see~\citep{boucheron2013concentration} for a comprehensive survey on the topic).
Perhaps the most popular two methods used in learning theory are the martingale method~\citep{azuma1967weighted,mcdiarmid1998concentration} 
and the information-theoretic entropy method~\citep{boucheron2003concentration,maurer2019bernstein}.
Both of these give many well-known and useful inequalities:
the first family includes the celebrated Azuma-Hoeffding and so-called \emph{bounded-differences} inequalities popularized by~\cite{mcdiarmid1998concentration}, while the second family is mostly known for powerful exponential Efron-Stein inequality which allows to state many prominent concentration inequalities as its special case (for instance, inequalities for self-bounding functions and Talagrand's convex distance inequality).

Roughly speaking, a common feature of both families is that
they relate concentration of $\Delta$ around zero to 
the sensitivity of $f$ to coordinatewise perturbations, 
expressed through the \emph{\acf{ES} variance proxy}
\begin{equation}
  \label{eq:VES}
  \VES = \E\br{\sum_{k=1}^n (f(S) - f(S\repk))_+^2 \,\middle|\, S}\,,
\end{equation}
where $(s)_+ = \max\cbr{0,s}$ and notation $S\repk$ means that the $k$th element of $S$ is replaced by $X_k'$, where $S'=\pr{X_1', X_2', \ldots, X_n'}$ is an independent copy of $S$.

For example,
an inequality closely related to
the well-known bounded-differences (or McDiarmid's) inequality follows from a conservative upper bound on~\eqref{eq:VES}: assuming that $\VES \leq c$ almost surely for some positive constant $c$, one has
\[
  \P\pr{\abs{\Delta} \leq \sqrt{2 c x}} \geq 1 - e^{-x}\,, \qquad x \geq 0~.
\]
This inequality is of course rather pessimistic since it neglects information about moments of $\Delta$.
A tangible step forward in proving less conservative inequalities was done in the context of so-called entropy method.
In particular, one of the central achievements of the entropy method is the following `exponential \ac{ES} inequality':\footnote{Recall that the concentration inequality follows from the Chernoff bound, i.e. $\P(\Delta \geq t) \leq \inf_{\lambda \in (0,1)} \E[e^{\lambda \Delta - \lambda t}].$}
\begin{equation}
  \label{eq:exp_ES}
  \ln \E\br{e^{\lambda \Delta}} \leq \frac{\lambda}{1 - \lambda} \ln \E\br{e^{\lambda \VES}}\,, \qquad \lambda \in (0, 1)~.
\end{equation}
This inequality bounds the \ac{MGF} of $\Delta$ through the \ac{MGF} of its \ac{ES} variance proxy, implying that by controlling the latter, one can obtain tail bounds involving moments.
For instance,
if for any choice of the distribution $\sD$, $f$ satisfies
$\VES \leq a f(S) + b$ for some constants $a,b > 0$, it is called a \emph{weakly self-bounding function} and
we can employ \cref{eq:exp_ES} to show that the first moment of $f$ and constants $a,b$ control the tail behavior of $\Delta$:
\begin{equation}
  \label{eq:self_bounding}
  \P\pr{\Delta \leq 2 \sqrt{\pr{a \E[f(S)] + b} x} + 2 a x} \geq 1 - e^{-x}\,, \qquad x \geq 0~.
\end{equation}
Thus, whenever $a$ is decreasing in $n$ (for instance, when $f$ is an average of random variables), one gets a dominating lower-order term when the first moment is small enough.
This behavior is reminiscent of the classical Bernstein's inequality and proved to be useful in a number of applications, such as generalization bounds with localization~\citep{bartlett2002localized,srebro2010smoothness,catoni2007pacbayesian} and empirical Bernstein-type inequalities~\citep{maurer2009empirical,tolstikhin2013pac}.

It is natural to ask whether we can get similar inequalities with higher order moments.
Indeed, a recent line of work by~\cite{maurer2019bernstein,maurer2018empirical} introduced
Bernstein-type concentration inequalities for general functions where in place of the variance, one has an expected \ac{ES} variance proxy (note that one still controls the variance of $f$ indirectly thanks to the \acl{ES} inequality $\Var(f) \leq \E[\VES]$).
However, this comes at a cost of controlling the first and the second moments of $\VES$.
Formally, if for any choice of distribution $\sD$, $f$ satisfies\footnote{  
  Subscript $-k$ in $\E_{-k}[\cdot]$ stands for conditioning on everything except for $X_k$.\\
  Notation $s^{(k)}$ stands for replacement of $k$th element of $s$ with the $k$th element of $s'$ and $(k,j)$ stands for replacement of both $k$th and $j$th elements with their counterparts in $s'$.}
\begin{align*}
\max_{k \in [n]}f(S) - \E_{-k}[f(S)] \leq b~,\
\sup_{s, s' \in \sZ} \sum_{k,j:k \neq j} \pr{ \big(f(s) - f(s\repk)\big) - \big(f(s\repj) - f(s\repkj) \big) }^2 \leq a^2
\end{align*}
almost surely for some $a, b > 0$,
we have a tail bound
\begin{equation}
  \label{eq:maurer_es}
  \P\pr{ \Delta \leq \sqrt{2 \E\br{\VES} x} + \big(\sqrt{2} a + \tfrac{2}{3} b\big) x } \geq 1 - e^{-x} \qquad x \geq 0~.
\end{equation}
Thus, to have a Bernstein-type behavior of the bound, $a$ and $b$ should be of a lower order.
Note the connection of \cref{eq:maurer_es} to the exponential \ac{ES} inequality~\ref{eq:exp_ES}: the condition on $f$ outlined above is sufficient to control the second and higher order moments of $\VES$, thus giving a concentration inequality for $\Delta$.

Despite their generality, all of these bounds implicitly control moments of $\VES$, which makes them difficult to apply in some cases.
The pair $(a,b)$ cannot depend on the sample and typically one would require boundedness of $f$ to obtain a well-behaved $(a,b)$.
One way to avoid these limitations is to revisit exponential ES inequality and analyze \ac{MGF} of \ac{ES} variance proxy in an application-specific way (for example, to assume a subexponential behavior of $\VES$)
\citep{abou2019exponential}.
However, in general this requires knowledge of additional parameters (such as scale and variance factor of the underlying subexponential distribution).
\subsection{Our Contribution}
\paragraph{Semi-empirical Efron-Stein Inequalities.}
In this paper we prove concentration inequalities without making apriori assumptions on the moments of the \ac{ES} variance proxy and instead we state bounds on the upper tail probability in terms of the \emph{semi-empirical \ac{ES}} variance proxy
\begin{equation}
  V = \sum_{k=1}^n \E\br{(f(S) - f(S\repk))^2 \,\middle|\, X_1, \ldots, X_k}\,.
\end{equation}
Note that $V$ is semi-empirical since it depends on both distribution $\sD$ and sample $S$.
Another property of $V$ is that it is \emph{asymmetric} w.r.t.\ the sample $S$ --- in general $V$ depends on the order of elements in the sequence $(X_1, X_2, \ldots, X_n)$, due to conditional expectation.
However, as we discuss in the following section (see applications), this does not affect sums and weakly affects normalized sums.

Our first result (\cref{thm:stability_vrep}) gives an exponential bound
\begin{equation}
  \label{eq:intro:es_1}
  \P\pr{|\Delta| \leq \sqrt{2 (\Vtilrep + y) \pr{1 + \ln(\sqrt{1 + \Vtilrep/y})} x}} \geq 1 - e^{-x}\,, \qquad x \geq 2,\ y > 0~.
\end{equation}
This inequality does not require boundedness of random variables, nor of $f$ --- the only crucial assumption is independence of elements in $S$ from each other.
Observe that \cref{eq:intro:es_1} essentially depends on $V$ and a positive free parameter $y$, which must be selected by the user. 
For instance, a problem agnostic choice of $y = 1/n^2$ for any $x \geq 2$ gives us w.p.\ (with probability) at least $1-e^{-x}$,
\[
  |\Delta| \leq \sqrt{2 (\Vtilrep + 1/n^2) \pr{1 + \ln(\sqrt{1 + n^2 \Vtilrep})} x}~.
\]
This recovers a Bernstein-type behavior, that is, the dominance of the lower-order term whenever $V$ (a variance proxy) is small enough.
The price we pay for such a simple choice of $y$ is a logarithmic term.
In general, one can achieve even sharper bound if the range of $\Vtilrep$ is known (or can be guessed) --- in this case, we can take a union bound over some discretized range of $y$, and select $y$ minimizing the bound.
In addition, we show a version of the bound that does not involve $y$ and thus it is scale-free. This version of our inequality, however, depends on $\E[V]$:
\[
  \P\pr{|\Delta| \leq 2 \sqrt{(\Vtilrep + \E[\Vtilrep]) x}} \geq 1 - \sqrt{2} e^{-x}\,, \qquad x \geq 0~.
\]

\paragraph{PAC-Bayesian Semi-Empirical Efron-Stein Inequalities.}
So far we have presented concentration inequalities which hold for fixed functions $f$.
However, in many learning-theoretic applications we are interested in concentration w.r.t.\ the class of functions, for example when $f$ potentially depends on the data.
In the following we extend our results to the class $\sF(\Theta) \equiv \cbr{f_{\vartheta} \,:\, \sZ \to \reals \,\middle|\, \vartheta \in \Theta}$, where $\Theta$ is some parameter space.
We focus on the stochastic, \emph{PAC-Bayesian} model, where functions are parameterized by a random variable $\param \sim \ps$
given some probability kernel $\wh{p}$ from $\sZ$ to $\Theta$
\footnote{Given $z\in \sZ$, we will use abbreviation $\hat{p}_z(\cdot)$ to denote $\hat{p}(\cdot|z)$.}.
For example, in the statistical learning setting, the predictor is parameterized by $\param$ sampled from $\ps$ called the \emph{posterior}, while $f_{\param}$ represents an empirical loss of the predictor (we discuss this in the upcoming section).
In particular, defining a $\theta$-dependent deviation
\[
  \Delta_{\param}
  =  f_{\param}(S) -\int f_{\param}(s) \, \sD(\diff s)
\]
and a semi-empirical \ac{ES} variance proxy
\[
  V_{\param} = \sum_{k=1}^n \E\br{(f_{\param}(S) - f_{\param}(S\repk))^2 \,\middle|\, \param, X_1, \ldots, X_k}\,,
\]
we show (in \cref{thm:pac_bayes_self_normalized_concentration},~\cref{sec:paces}) that for an arbitrary probability kernel $\wh{p}$ from $\sZ$ to $\Theta$ and an arbitrary probability measure $p^0$ on $\Theta$ called the \emph{prior},
w.p.\ at least $1-e^{-x}$ for any $x \geq 2$, $y > 0$,
\begin{equation}
  \label{eq:intro:es_pac_bayes}
  |\E[\Delta_{\param} \,|\, S]|  
  \leq
  \sqrt{
    2 \pr{\E[V_{\param} \,|\, S] + y}
    \pr{\KL\pr{\ps \,||\, p^0} + x + x \ln\pr{\sqrt{1 + \E[V_{\param} \,|\, S]/y}}}
  }~,
\end{equation}
where the \ac{KL} divergence between $\ps$ and $p^0$ (assuming that $\ps \ll p^0$) captures the \emph{effective} capacity of $\Theta$ under respective measures.
Similarly as before, we also have a $y$-free version, which holds w.p.\ at least $1-2 e^{-x}$ for any $x \geq 0$:
\[
  |\E[\Delta_{\param} \,|\, S]|
  \leq
  \sqrt{2 (\E[V_{\param}] + \E[V_{\param} \,|\, S]) \pr{\KL\pr{\ps \,||\, p^0} + 2 x}}~.
\]
Once again, these results do not require boundedness of random variables, nor of $f \in \sF(\Theta)$, and the concentration is essentially controlled by the \emph{expected} variance-proxy $\E[V_{\param} \,|\, S]$.
Next, we discuss several specializations of our results and note several key connections to the literature on the PAC-Bayesian analysis.
\subsection{Applications}
\label{sec:intro:applications}
Now we discuss several applications of our inequalities.
Throughout this section we assume that inequalities hold for any $x \geq 2$ and any $y > 0$, unless stated otherwise.
\paragraph{Generalization bounds.}
PAC-Bayesian literature often discusses bounds on the \emph{generalization gap}, which is a special case covered by our results.
In this scenario, $f_{\param}$ is defined as an average of some non-negative \emph{loss} functions, incurred by the predictor of interest parameterized by $\param$ on a given \emph{example}.
Here, $\param$ is sampled from the posterior $\ps$ (a density over the parameter space $\Theta$) chosen by the learner.
In particular, let $\sZ_1 = \dots = \sZ_n$ and let $\ell \,:\, \Theta \times \sZ_1 \to R$ be some loss function with co-domain $R \subseteq \reals_+$.
Then, we define the \emph{population loss} and the \emph{empirical loss} as
\[
  L(\param) = \E[\ell(\param, X_1')] \qquad \text{and} \qquad
  \wh{L}_S(\param) = \frac1n \sum_{k=1}^n \ell(\param, X_k)~,
\]
respectively, and taking $f_{\param}(S) = \wh{L}_S(\param)$, the generalization gap is defined as $\Delta_{\param} = \wh{L}_S(\param) - L(\param)$.

The vast majority of PAC-Bayesian literature, e.g.~\citep{mcallester1998some,seeger2002pac,langford2003pac,maurer2004note} assume that the loss function is bounded, i.e.\ w.l.o.g.\ $R \equiv [0, 1]$.
In such case, $V_{\param} \leq 1 / n$ and taking $y = 1/n$, \cref{eq:intro:es_pac_bayes} immediately implies that w.p.\ at least $1-e^{-x}$, 
\[
  |\E[\Delta_{\param} \,|\, S]|  
  \leq
  2 \sqrt{
    \frac1n \pr{\KL\pr{\ps \,||\, p^0} + x (1 + \ln(\sqrt{2}))}
  }~.
\]
This basic corollary tightens classical results by replacing term $\ln(2 \sqrt{n})$ with a universal constant, but slightly looses in terms of a multiplicative constant.
The technical assumption on boundedness of the loss is not easy to avoid and the usual way to circumvent this would be to resort to a sub-exponential behavior of the relevant quantities, such as an empirical loss or a generalization gap~\citep{alquier2016properties,germain2016pac}.
Recently, few works have also looked into the PAC-Bayesian analysis for heavy-tailed losses:~\cite{alquier2018simpler} proposed a polynomial moment-dependent bound with $f$-divergence, while~\cite{holland2019pac} devised an exponential bound which assumes that the second (uncentered) moment of the loss is bounded by a constant.

Here, without any of those assumptions, we obtain (\cref{cor:gen_bound}) a high-probability semi-empirical generalization bound for unbounded loss functions
\footnote{Notation $\leqCln$ suppresses universal constants and logarithmic factors, while $\leqC$ supresses only universal constants.}
($R\equiv[0,\infty)$),
\begin{equation}
  \label{eq:intro:var_bound}
  |\E[\Delta_{\param} \,|\, S]|
  \leqCln
  \sqrt{\frac1n \E\br{\frac1n \sum_{k=1}^n \Big(\ell(\param, X_k)^2 + \ell(\param, X_k')^2\Big) \,\middle|\, S} \KL(\ps\,||\,p^0)} + \frac1n~,
\end{equation}
where $\E[V_{\param} \,|\, S] \leq (1/n^2) \E\br{\sum_{k=1}^n \ell(\param, X_k)^2 + \ell(\param, X_k')^2 \,\middle|\, S}$.
This result is close in spirit to the \emph{localized} bounds of \cite{catoni2007pacbayesian,langford2003pac,tolstikhin2013pac}: for a small variance proxy (here sum of squared losses) we get the dominance of a lower-order term $1/n$. %

Finally, while the bound of \cref{eq:intro:var_bound} is semi-empirical (note that we condition only on $X_k$), by additionally assuming boundedness of the loss,
it implies a fully empirical result (\cref{thm:bounded_losses}):
\[
  |\E[\Delta_{\param} \,|\, S]|
  \leqCln
  \sqrt{\frac{1}{n}\, \E\br{\frac1n \sum_{k=1}^n \ell(\param, X_k)^2 \,\middle|\, S}\, \KL\pr{\ps \,||\, p^0}} + \frac{\KL\pr{\ps \,||\, p^0}}{n} + \frac1n~.
\]
Such \emph{empirical Bernstein bounds}~\citep{audibert2007tuning,maurer2009empirical} in PAC-Bayesian setting were first investigated by~\cite{tolstikhin2013pac}.
The bound we present here is similar to the one of~\cite{tolstikhin2013pac}, but slightly differs since we consider the sum of squared losses (with co-domain $[0,1]$) rather than the sample variance.
Nevertheless, we recover a similar behavior, that is a ``fast'' order $(1 + \KL(\ps \,||\, p^0))/n$ for the small enough variance proxy.
\paragraph{Off-policy Evaluation with \acf{WIS}.}
Consider the stochastic decision making model where the pair of random variables called the \emph{action-reward pair} is distributed according to some unknown joint probability measure $\sD \in \sM_1([K] \times [0,1])$.
In such model, also known as a \emph{stochastic bandit feedback model} (see~\citep{lattimore2018bandit} for a detailed treatment on the subject), an agent takes action $A$ by sampling it from a discrete probability distribution called the \emph{target policy} $\pi \in \sM_1([K])$ and observes a realization of reward $R \sim \sD(\cdot \,|\, A)$.
In the \emph{off-policy} setting of this model, the learner only observes a tuple of actions and rewards $S = \pr{(A_1, R_1),\ldots, (A_n, R_n)}$ generated by sampling each action $A_i$ from another fixed discrete probability distribution $\pi_b \in \sM_1([K])$ called the \emph{behavior} policy, while corresponding rewards are distributed as $R_k \sim \sD(\cdot \,|\, A_k)$.

In the \emph{off-policy evaluation} problem, our goal is to estimate an expected reward, or the \emph{value} of a fixed target policy $\pi$,
\[
  v(\pi) = \sum_{a \in [K]} \pi(a) \E[R | A=a]~,
\]
by relying on $S$, where actions and rewards are collected by policy $\pi_b$.
Since observations are collected by another policy, we face a distribution mismatch problem and an estimator of the value
is typically designed using a variant of an \emph{importance sampling}, while aiming to maintain a good bias and variance trade-off.
In this paper we study \emph{\acf{WIS}} (or \emph{self-normalized importance sampling}) estimator
\[
  \wh{v}\WIS(\pi) = \frac{\sum_{k=1}^n W_k R_k}{\sum_{i=1}^n W_i}
\]
where \emph{importance weights} $W_k$ are defined as a ratio of policies given an action, $W_k = \pi(A_k)/\pi_b(A_k)$.
\ac{WIS} estimator is known for a relatively low variance in practice~\citep{hesterberg1995weighted},
yet it concentrates well even when importance weights are unbounded,
since all of its moments are bounded, which makes it appealing for confidence estimation.

In this paper we show a lower bound on the value of the target policy when employing \ac{WIS}, which partially captures the variance of an estimator.
We prove (\cref{thm:off_policy_eval}) that  w.p.\ at least $1-(n+1)e^{-x}$,
\begin{align*}
  &v(\pi)
  \geq
  \frac{N_x(n)}{n}
  \pr{\wh{v}\WIS(\pi)
    - \sqrt{2 (2 V\WIS + y) \pr{1 + \ln\pr{\sqrt{1 + 2 V\WIS / y}}} x}}_+\\
  \text{where} \qquad &V\WIS = \frac{1}{N_x(n)^2} \sum_{k=1}^n \pr{W_k^2 + \E[W_k^2]}
                        \quad \text{and} \quad
                        N_x(n) = \pr{n - \sqrt{2 \conf n \E\br{W_1^2}}}_+~.
\end{align*}
Here $V\WIS$ acts as a variance proxy and can be easily computed since the distribution of the importance weights is known.
Note that the bound can be further improved by taking a tighter variance proxy (\cref{cor:wm},~\cref{prop:V_bounds}) at an additional computational cost.
Computationally efficient version of $V\WIS$ we discuss here states the rate of concentration in terms of the variance of the importance weights.

Presented high-probability results do not require boundedness of importance weights, nor any form of truncation, prevalent in the literature on (weighted) importance sampling~\citep{swaminathan2015self,BoPe13,thomas2015high_a}.
Indeed, it is not hard to see that by truncating the weights we can apply standard concentration inequalities.
However, truncation biases the estimator and in practice requires to carefully tune the level of truncation to guarantee a good bias-variance trade-off.
We avoid this precisely because our concentration inequalities do not require control of higher moments through boundedness.
While another general Bernstein-type concentration inequalities (e.g.\ \cref{eq:self_bounding,eq:maurer_es}) could be used for such problems, they would require truncation of importance weights.

\paragraph{Off-policy Learning with \ac{WIS}.}
An \emph{off-policy learning} problem is a natural extension of the evaluation problem discussed earlier.
Here, instead of the evaluation of a fixed target policy, our goal is to select a policy from a given class, which maximizes the value.
In this paper we propose a PAC-Bayesian lower bound on the value by specializing our \ac{ES} PAC-Bayesian inequalities. In particular, we consider a class of parametric target policies $\cbr{\pi_{\vartheta} \in \sM_1([K]) \,:\, \vartheta \in \Theta}$ for some parameter space $\Theta$.
Similarly as before, we assume that the parameter $\param$ is sampled from the posterior $\ps$ (typically, chosen by the learner after observing the data), where $\wh{p}$ is some probability kernel $\wh{p}$ from $([K] \times [0,1])^n$ to $\Theta$.
Note that the probability measure $\ps$ depends on the tuple of observed action-reward pairs $S$ generated as described before, and importance weights are now defined w.r.t.\ the random parameter $\param$, that is $W_{\param,k} = \pi_{\theta}(A_k) / \pi_b(A_k)$.

We show (\cref{thm:offpollearn}) that for an arbitrary probability kernel $\hat{p}$ from $([K] \times [0,1])^n$ to $\Theta$ and an arbitrary probability measure $p^0$ over $\Theta$,  w.p.\ at least $1-2e^{-x}$,
\[
  \E[v(\pi_{\param}) \mid S]
  \geq
  \pr{
  \E[\wh{v}\WIS(\pi_{\param})  \mid S]
  - \E\br{ \abs{\frac{n}{N_{\param, x}(n)} - 1} \,\middle|\, S}
  - \sqrt{
  2 \pr{y + \E[V\WIS_{\param} \mid S]}
  C_{x,y}(S)
  }
  }_+
\]
where the effective capacity of the policy class is represented by the \ac{KL} divergence in
\[
  C_{x,y}(S) = \KL\pr{\ps \,||\, p^0} + x + x \ln\pr{\sqrt{1 + \E[V\WIS_{\param} \,|\, S] / y}}~,
\]
and the variance proxy of an estimator is
$
  V\WIS_{\param} = \sum_{k=1}^n \E[\tilde W_{\param,k}^2 + \tilde U_{\param,k}^2 \,|\, \param, A_1, \ldots, A_k ]~,
$
\[
  \text{where} \quad
  \tilde W_{\param,k} = W_{\param,k} / (W_{\param,1} + \dots + W_{\param,n})~, \quad \tilde U_{\param,k} = W_{\param,k}'/(W_{\param,1} + \dots + W_{\param,k}' + \dots + W_{\param,n})~.
\]
Here $\E[ \abs{n/N_{\param, x}(n) - 1} \,|\, S]$ captures the \emph{bias} of an estimator, while $V\WIS_{\param}$ is the \emph{variance} proxy.
Similarly as in case of evaluation, the variance proxy is not fully empirical, however it can be computed exactly since the distribution of importance weights is known (it is given by the behavior policy $\pi_b$ and the target policy $\pi_{\param}$).
The variance proxy presented here is also closely related to the 
\emph{\ac{ESS}} encountered in the Monte-Carlo theory~\citep[Chap.9]{owen2013book},~\citep{elvira2018rethinking}, which provides a problem-dependent convergence rate of the \ac{WIS} estimator.
When all importance weights are equal to one (perfectly matching policies), \ac{ESS} is of order $1/n$; on the other hand \ac{ESS} approaches $1$ when importance weights concentrate on a single action.
The role of \ac{ESS} in a variance-dependent off-policy problems was also observed by~\cite{metelli2018policy}, although in a context of polynomial bounds for fixed target policies.

Thus, maximizing presented lower bound w.r.t.\ a (parametric) probability measure $\ps$ gives a way to learn a target policy maximizing the value, while maintaining a bias-variance trade-off of \ac{WIS} estimator.
This idea is well-known in the off-policy learning literature ---
typically this is done through empirical Bernstein bounds by employing importance sampling estimator~\citep{thomas2015high_b} and sometimes \ac{WIS} estimator~\citep{swaminathan2015self}, however, all these techniques require some form of weight truncation.
As in the off-policy evaluation case, the trade-off between the bias and the variance has to be carefully controlled by tuning the level of truncation.
Our results provide an alternative route, free from an additional tuning.
A truncation-free uniform convergence bounds for importance sampling were also derived by~\cite{cortes2010learning}, however here we explore a PAC-Bayesian analysis approach.
\subsection{Proof Ideas}

Concentration inequalities shown in this paper to some extent are based on the inequalities for \emph{self-normalized} estimators by \cite{pena2008self}.
In particular, we use inequalities derived through the \emph{method of mixtures}~\cite[Chap.\ 12.2.1]{pena2008self}, which hold for the pair of random variables $A \in \reals, B > 0$ satisfying the condition
$\sup_{\lambda \in \reals}\E\br{\exp\pr{\lambda A - \lambda^2 B^2 / 2}} \leq 1~.$
Such random variables are called a \emph{canonical pair} and our semi-empirical \ac{ES} inequalities follow by proving that $(\Delta, \sqrt{V})$ indeed forms a canonical pair.
We do so by applying exponential decomposition inspired by the proof of the Azuma-Hoeffding's inequality to the condition stated above, while $\Delta$ is represented by the Doob martingale difference sequence.
Note that a similar technique was also applied by~\cite{rakhlin2017equivalence} in the context of the self-normalized martingales.

Our PAC-Bayesian inequalities follow a more involved argument. As in the classical PAC-Bayesian analysis we start from the Donsker-Varadhan change-of-measure inequality applied to the function $f(\param) = \lambda \E[\Delta_{\param} - (\lambda^2/2) V_{\param} \,|\, S]$ for $\lambda \in \reals$, and note that the log-\ac{MGF} of $f$ at the prior parameter is bounded by $1$ due to the canonical pair argument.
The rest of the proof is dedicated to the tuning of $\lambda$.
One possibility here would be to apply the union bound argument of~\cite{seldin2012pac}, however, since $\lambda$ is unbounded, this allows us to take a more straightforward path.
In particular, we employ the method of mixtures (used in \citep[Chap.\ 12.2.1]{pena2008self} to derive the aforementioned inequalities)
to achieve analytic tuning of the bound w.r.t.\ $\lambda$.
The idea behind the method of mixtures is to integrate the parameter of interest under some analytically-integrable probability measure.
The choice of the Gaussian density with variance $y^2$ (recall that $y$ is a free parameter) and the Gaussian integration w.r.t.\ $\lambda$ leads to the concentration inequalities.
To the best of our knowledge, this is the first application of the method of mixtures in PAC-Bayesian analysis, which is an alternative technique to analytical (union bound)~\citep{seldin2012pac,tolstikhin2013pac} and empirical tuning of $\lambda$ as in \citep{thiemann2016strongly}.

Finally, described applications follow from the analysis of the semi-empirical \ac{ES} variance proxy.
In case of the generalization error, such analysis is straightforward --- our main bounds are obtained by observing multiple cancellations in $\wh{L}_S(\param) - \wh{L}_{S\repk}(\param)$.
The case of \ac{WIS}, this comes by the \emph{stability} analysis of the estimator --- given the removal of the $k$-th importance weight, the difference of estimators is bounded by $W_k / (W_1 + \dots + W_n)$.

\subsection{Additional Related Work}
Our PAC-Bayesian bounds are related to the martingale bounds of~\cite{seldin2012pac}.
In particular, our results need to be compared to the PAC-Bayes-Bernstein inequality for martingales~\cite[Theorem 8]{seldin2012pac}.
In principle, their inequality could be applied to the Doob martingale difference sequence to prove a concentration bound.
However, this would hold only for the \emph{bounded} difference sequences, restricted family of probability kernels, and would yield inequality with a different 'less empirical' \ac{ES} variance proxy (with conditioning up to $k-1$ elements in expectation).
The technique of~\cite{seldin2012pac,tolstikhin2013pac,thiemann2016strongly} at its heart relies on the self-bounding property of the variance proxy to control the log-\ac{MGF} term arising due to the PAC-Bayesian analysis.
This control is possible thanks to inequalities obtained through the entropy method.
On the other hand, self-bounding property in these cases holds for a limited range of $\lambda$, and the method of mixtures applied in our proofs cannot be used here without introduction of superfluous error terms (because of the clipped Gaussian integration).
Another direction in controlling log-\ac{MGF}, related to the empirical Bernstein inequalities for martingales was explored in the online learning literature~\citep{cesa2007improved,wintenberger2017optimal} by linearization of $x \to e^{x-x^2}$.

PAC-Bayesian analysis in learning theory is not restricted to the generalization gap discussed in~\cref{sec:intro:applications}.
Several works~\citep{maurer2004note,seeger2002pac} have investigated generalization by giving upper bounds on the \ac{KL} divergence of a Bernoulli variable (assuming that loss function is bounded on $[0,1]$), which are
clearly tighter than the difference bounds (due to the Pinsker's inequality).
In this paper we forego this setting for the sake of generality, however we suspect that KL-Bernoulli \ac{ES} bounds can be derived for the bounded loss functions.

A number of works have also looked into PAC-Bayesian bounds (and bounds for the closely related \emph{Gibbs} predictors) on the \emph{excess risk} $\E[L(\param) \,|\, S] - \inf_{\vartheta \in \Theta} L(\vartheta)$,~e.g.\ \citep{catoni2007pacbayesian,alquier2016properties,kuzborskij2019distribution,grunwald2019tight}.
A tantalizing line of research would be to investigate the use of our semi-empirical results in the context of an excess risk analysis.

Finally, in recent years many works~\citep{dziugaite2017computing,neyshabur2018pac,rivasplata2018pac,mhammedi2019pacbayes} have observed that PAC-Bayesian bounds tend to give less conservative numerical estimates of the generalization gap for neural networks compared to the alternative techniques based on the concentration of empirical processes.
Semi-empirical bounds proposed in this paper offer opportunities for sharpening these results in a data-dependent fashion.
\paragraph{Organization.}
In \cref{sec:concentration} we prove concentration inequalities for the fixed function, where the proof crucially relies on \cref{lem:canonical}, whose proof deferred to the \cref{sec:proofs_semi_empirical}.
In \cref{sec:paces} we present the PAC-Bayesian extension of the concentration inequalities and present a major part of its proof, which is one of the main contributions of this paper.
Finally, proofs for generalization bounds are presented in \cref{sec:gen_bounds}, while proofs for the \ac{WIS} value bounds are presented in \cref{sec:wa} and \cref{sec:pac_bayesian_wis}.
\section{Preliminaries}
Throughout this paper, we use $f \leqC g$ to indicate that there exists a universal constant $C > 0$ such that $f \leq C g$ holds uniformly over all arguments.
We use $\leqCln$ to indicate a version of $\leqC$ which also supresses logarithmic factors.
Notation $(s)_+$ is used to denote the positive part of the real number $s\in \R$.
We use notation $\sM_1(\sA)$ to denote a family of probability measures supported on
a set $\sA$.
If $p$ and $q$ are densities over $\Theta$ such that $p \ll q$ and $X \sim p$, the \acf{KL} divergence between $p$ and $q$ is defined as
$\KL(p, q) = \E\br{\ln\big(p(X) / q(X)\big)}$.
We say that a random variable $X$ is $\nu$-subgaussian if
$\ln \E\br{e^{\lambda (X - \E[X])}} \leq \lambda^2 \nu$ for every $\lambda \in \reals$.

\section{Semi-Empirical Concentration Inequalities}
\label{sec:concentration}
In this section we prove \emph{semi-empirical} concentration inequalities.
Recall that we focus on measurable functions $f : \sZ \to \reals$ of a random tuple $S = \pr{X_1, X_2, \ldots, X_n}\sim \sD \in \sM_1(\sZ)$, where elements of $S$ are distributed independently from each other and $\sZ = \sZ_1 \times \dots \times \sZ_n$.
In this section we prove bounds on the upper tail probability of $\Delta = f(S) - \E[f(S)]$.
\begin{theorem}
  \label{thm:stability_vrep}
  Let the semi-empirical Efron-Stein variance proxy be defined as
  \begin{equation}
    V = \sum_{k=1}^n \E\br{(f(S) - f(S\repk))^2 \,\middle|\, X_1, \ldots, X_k}\,.
  \end{equation}
  Then, for any $\conf \geq 2$, with probability at least $1-e^{-\conf}$ and any $y > 0$,
  \[
    |\Delta| \leq \sqrt{2 (\Vtilrep + y) \pr{1 + \frac{1}{2} \ln\pr{1 + \frac{\Vtilrep}{y}}} x}~.
  \]
  In addition, for any $\conf>0$, with probability at least $1-\sqrt{2} e^{-\conf}$ we have
  \[
    |\Delta| \leq 2 \sqrt{(\Vtilrep + \E[\Vtilrep]) \conf}~.
  \]  
\end{theorem}
Note the similarity to the Efron-Stein inequality, which bounds the variance of $f(S)$ with $\E[\Vtilrep]$. 
The proof of \cref{thm:stability_vrep} combines the argument used in the proof of McDiarmid's/Azuma-Hoeffding's inequality 
with a concentration inequality due to \citet{pena2008self}. To state this inequality, recall that a pair $(A,B)$ of random variables is called a \textbf{canonical pair} if $B \ge 0$ and 
\begin{align}
    \label{eq:dom_condition}
	\sup_{\lambda\in \R} \E\br{\exp\pr{\lambda A - \frac{\lambda^2}{2} B^2}} \leq 1~.
\end{align}
The result of \citeauthor{pena2008self} shows that if $(A,B)$ is a canonical pair then $|A|$ has a (random) subgaussian behavior with variance proxy $B^2$: 
\begin{theorem}[Theorem 12.4 of \cite{pena2008self}]
  \label{thm:self_norm_concentration}
  Let $(A,B)$ be a canonical pair. Then, for any $t > 0$,
  \[
    \P\pr{\frac{|A|}{\sqrt{B^2 + (\E[B])^2}} \geq t} \leq \sqrt{2} e^{-\frac{t^2}{4}}~.
  \]
  In addition, for all $t \geq \sqrt{2}$ and $y > 0$,
  \[
    \P\pr{\frac{|A|}{(B^2 + y) \pr{1 + \frac{1}{2} \ln\pr{1 + \frac{B^2}{y}}}} \geq t} \leq e^{-\frac{t^2}{2}}~.
  \]
\end{theorem}
Thus, provided that $(\Delta,\sqrt{V})$ is a canonical pair, it is easy to see that 
\cref{thm:stability_vrep} follows from \cref{thm:self_norm_concentration} applied to $(\Delta,\sqrt{V})$.
Thus, it remains to be seen that $(\Delta,\sqrt{V})$ is a canonical pair,
which is established by the following lemma, whose proof is presented in~\cref{sec:proofs_semi_empirical}:
\begin{lemma}
  \label{lem:canonical}
  $(\Delta, \sqrt{V})$ is a canonical pair.
\end{lemma}

\section{PAC-Bayesian Bounds}
\label{sec:pacbayes}
In this section we prove a PAC-Bayesian version of~\cref{thm:self_norm_concentration}, which will consequently imply a PAC-Bayesian bound for classes of functions.

Let $\Theta$ denote an index set.
We call the collection $(\Delta_\theta,\sqrt{V_\theta})_{\theta\in \Theta}$
of pairs of random variables  indexed by elements of $\Theta$
a \textbf{canonical family}, if for each $\theta\in \Theta$, $(\Delta_\theta,\sqrt{V_\theta})$ is  
a canonical pair.
In this section we prove new PAC-Bayesian inequalities for families of canonical pairs that arise as functions of a common random element.
Below we let $\hat{p}$ to denote a probability kernel from $\sZ$ to $\Theta$.
Also, for brevity, given $z\in \sZ$, we will also use $\hat{p}_z(\cdot)$ to denote $\hat{p}(\cdot|z)$.

\begin{theorem}
  \label{thm:pac_bayes_self_normalized_concentration}
  For some space $\Theta$, let
$(\Delta_\theta,\sqrt{V_\theta})_{\theta\in \Theta}$ be a canonical family and $S\sim \sD$, jointly distributed with $(\Delta_\theta,\sqrt{V_\theta})_{\theta\in \Theta}$ and taking values in $\sZ$.
Fix an arbitrary probability kernel $\hat{p}$ from $\sZ$ to $\Theta$ and an arbitrary probability measure $p^0$ over $\Theta$, and let $\param \sim \ps$.
Then,  for any $x\ge 0$,
  with probability at least $1-2 e^{-x}$ 
  we have that
  \begin{equation}
    \label{eq:pac_bayes_mix_3}
    |\E[\Delta_{\param} \,|\, S]|
    \leq
    \sqrt{2 (\E[V_{\param}] + \E[V_{\param} \,|\, S]) \pr{\KL\pr{\ps \,||\, p^0} + 2 x}}~.
  \end{equation}
  In addition, for all $y > 0$ and $x \geq 2$ with probability at least $1-e^{-x}$ we have
  \begin{equation}
    \label{eq:pac_bayes_mix_4}
    |\E[\Delta_{\param} \,|\, S]|
    \leq
    \sqrt{
      2 \pr{y + \E[V_{\param} \,|\, S]}
      \pr{\KL\pr{\ps \,||\, p^0} + x + \frac{x}{2} \ln\pr{1 + \frac{1}{y} \E[V_{\param} \,|\, S]}}
    }~.
  \end{equation}
\end{theorem}
The proof (given in~\cref{sec:proofs_semi_empirical}) largely relies on the following theorem, which allows to bound a moment-generating function of a random variable
{\small $\sqrt{\pr{\E[\Delta_{\param} \,|\, S]^2/\pr{\E[V_{\param}] + \E[V_{\param} \,|\, S]} - 2 \KL\pr{\ps \,||\, p^0}}_+}$}.
Note that the crucial part is to show that this random variable is subgaussian (as show in~\eqref{eq:pac_bayes_mix_2}).
\begin{theorem}
  \label{thm:pac_bayes_self_normalized_mgf}
Under the same conditions as in  \cref{thm:pac_bayes_self_normalized_concentration},
for any $y > 0$, we have
  \begin{equation}
    \label{eq:pac_bayes_mix_1}
    \E\br{\frac{y}{\sqrt{y^2 + \E[V_{\param} \,|\, S]}} \, \exp\pr{\frac{\E[\Delta_{\param} \,|\, S]^2}{2 (y^2 + \E[V_{\param} \,|\, S])} - \KL\pr{\ps \,||\, p^0}}}
  \leq
  1~.
  \end{equation}
Furthermore, for all $x \geq 0$, 
\begin{equation}
  \label{eq:pac_bayes_mix_2}
  \E\br{
    \exp\pr{
      x \sqrt{
      \left(
      \frac{\E[\Delta_{\param} \,|\, S]^2}{\E[V_{\param}] + \E[V_{\param} \,|\, S]} - 2 \KL\pr{\ps \,||\, p^0}
      \right)_+
      }
    }
  }
  \leq 2 e^{x^2}~.
\end{equation}
\end{theorem}
The proof combines PAC-Bayesian ideas with the method of mixtures as described by~\cite{pena2008self} [Section 12.2.1].
\subsection{Proof of~\cref{thm:pac_bayes_self_normalized_mgf}}
  We start by applying the following \emph{change-of-measure} lemma, which is the basis of the PAC-Bayesian analysis.
  \begin{lemma}[\citet{DoVa75,DuEl97:weakconv,Gray11}]    
    Let $p, q$ be probability measures on $\Theta$ and let $X \sim p$, $Y \sim q$. Then, for any measurable function $f \,:\, \Theta \to \reals$ we have
    \[
      \E[f(X)] \leq \KL(p\,||\,q) + \ln \E\br{e^{f(Y)}}~.
    \]
  \end{lemma}
  The lemma with $p = \ps$, $q = p^0$, and $f(\param) = \lambda \Delta_{\param} - \frac{\lambda^2}{2} V_{\param}$ for a fixed $S$ implies
\begin{align*}
  \E\br{ \lambda \Delta_{\param} - \frac{\lambda^2}{2} V_{\param} \,\middle|\, S}
  &\leq
    \KL\pr{\ps \,||\, p^0}
    +
    \ln \E\br{ e^{\lambda \Delta_{\param^0} - \frac{\lambda^2}{2} V_{\param^0}} \,\middle|\, S}~.
\end{align*}
Exponentiation of both sides gives
\begin{align*}
  e^{\E\br{ \lambda \Delta_{\param} - \frac{\lambda^2}{2} V_{\param} \,\middle|\, S} - \KL\pr{\ps \,||\, p^0}}
  &\leq
    \E\br{ e^{\lambda \Delta_{\param^0} - \frac{\lambda^2}{2} V_{\param^0}} \,\middle|\, S}
\end{align*}
and taking expectation we have
\begin{align*}
  \E\br{e^{\E\br{ \lambda \Delta_{\param} - \frac{\lambda^2}{2} V_{\param} \,\middle|\, S} - \KL\pr{\ps \,||\, p^0}}}
  &\leq
    \E\br{ e^{\lambda \Delta_{\param^0} - \frac{\lambda^2}{2} V_{\param^0}}}
  =
    \E\br{ \E\br{ e^{\lambda \Delta_{\param^0} - \frac{\lambda^2}{2} V_{\param^0}} \, \middle| \, \param^0 } }
  \leq
    1\,,
\end{align*}
since $(\Delta_{\param^0}, \sqrt{V_{\param^0}})$ is a canonical pair for a fixed $\param^0$ by assumption.
Now we apply the method of mixtures with respect to the Gaussian distribution.
Multiplying both sides by $e^{-\lambda^2 y^2/2} y / \sqrt{2 \pi}$ for some $y > 0$, integrating w.r.t.\ $\lambda \in \reals$, and applying Fubini's theorem gives
\begin{align*}  
  \E\br{
  e^{- \KL\pr{\ps \,||\, p^0}}
  \frac{y}{\sqrt{2 \pi}} \int_{-\infty}^{\infty} e^{
  \lambda \E\br{\Delta_{\param} \,|\, S } - \frac{\lambda^2}{2} \E\br{V_{\param} \,\middle|\, S} - \frac{\lambda^2}{2} y^2
  } \diff \lambda
  }
  \leq 1~.
\end{align*}
We perform Gaussian integration and arrive at
\begin{equation}
  \label{eq:integrated_mixture}
  \E\br{\frac{y}{\sqrt{y^2 + \E[V_{\param} \,|\, S]}} \, \exp\pr{\frac{\E[\Delta_{\param} \,|\, S]^2}{2 (y^2 + \E[V_{\param} \,|\, S])} - \KL\pr{\ps \,||\, p^0}}}
  \leq
  1\,,
\end{equation}
which finishes the proof of~\cref{eq:pac_bayes_mix_1}.
For the second part, we consider the following standard lemma:
\begin{lemma}\label{lem:sgfromexp}
Let $U$ be a nonnegative valued random variable such that $a =\E\br{\exp(U^2/4)}$ is finite. Then, for any $x\ge 0$,
$\E\br{\exp(x U )} \le a e^{x^2}$ holds.
\end{lemma}
Setting 
$U=\sqrt{
\left(\E[\Delta_{\param} \,|\, S]^2 / \pr{\E[V_{\param}] + \E[V_{\param} \,|\, S]} - 2 \KL\pr{\ps \,||\, p^0}\right)_+
}$,
the lemma gives \cref{eq:pac_bayes_mix_2} provided that we show that $\E\br{\exp\pr{U^2/4}} \le 2$.
For this, let $y>0$. 
Introduce the abbreviations 
$A=y/\sqrt{y^2 + \E[V_{\param} \,|\, S]}$,
$B=\E[\Delta_{\param} \,|\, S]^2 / (2 y^2 + 2 \E[V_{\param} \,|\, S]) - \KL\pr{\ps \,||\, p^0}$
so that $\E\br{\exp\pr{U^2/4}} = \E\br{\exp(\pr{B}_+/2)}$. Note that $A>0$.
By Cauchy-Schwartz,
\begin{align}
\E\br{\exp(\pr{B}_+/2)}=
  \E\br{\exp\pr{\pr{B}_+/2} A^{1/2} A^{-1/2} }
  \le
  \sqrt{\E\br{A \exp\pr{\pr{B}_+} }} \, \sqrt{ \E\br{ A^{-1} } }  \,.
     \label{eq:proving_pac_bayes_mix_2:1}
\end{align}
Observe that $A \in (0,1]$ a.s. and that $\E[A \exp(B)] \leq 1$ by~\cref{eq:pac_bayes_mix_1}.
Now, we have
{\small
\begin{align*}
  \sqrt{\E\br{A \exp\pr{(B)_+}} \E\br{\frac{1}{A}}}
  &=
    \sqrt{\Big(
    \E\br{A \, \ind{B \geq 0} \exp\pr{B}} + \E\br{A \, \ind{B < 0}}
    \Big)
    \E\br{\frac{1}{A}}
    }
  \leq
    \sqrt{\E\br{\frac{2}{A}}}\,,
\end{align*}
}
and finally, by subadditivity of $\sqrt{\cdot}$ and Jensen's inequality,
\[
  \sqrt{\E\br{\frac{2}{A}}}
  =
    \sqrt{2 \E\br{\sqrt{\frac{y^2 + \E[V_{\param} \,|\, S]}{y^2}}}}
  \leq
    \sqrt{2 + 2 \frac{\sqrt{\E[V_{\param}]}}{y}}
  \leq
    2~,
\]
where the last inequality follows by taking $y = \sqrt{\E[V_{\param}]}$.
Thus, applying Lemma~\ref{lem:sgfromexp} with $a=2$ completes the proof.
\subsection{PAC-Bayesian Efron-Stein Inequalities}
\label{sec:paces}
Now, we apply \cref{thm:pac_bayes_self_normalized_concentration} to get concentration inequalities for classes of functions.
In particular, consider the class of functions parametrized by some space $\Theta$, that is $\sF(\Theta) \equiv \cbr{f_{\vartheta} \,:\, \sZ \to \reals_+ \,\middle|\, \vartheta \in \Theta}$.
Furthermore, we will assume that $\param \sim \hat{p}(\cdot|S)$, where $\hat{p}$ is a probability kernel from $\sZ$ to $\Theta$
and that $\param^0 \sim p^0$ where $p^0 \in \sM_1(\Theta)$ is a probability distribution over $\Theta$.
We let $S'$ be a random element that shares a common distribution with $S$ and which is independent of $(S,\theta)$.
Finally, we are interested in
bounds on the deviation $\E[\Delta_{\param} \,|\, S]$, where
\[
  \Delta_{\param}
  =  f_{\param}(S) -\int f_{\param}(s) \, \sD(\diff s)~,
\]
which hold simultaneously for any choice of $\ps$ and $p^0$,
and which are controlled by the $\param$-dependent version of a semi-empirical Efron-Stein variance proxy
\[
  V_{\param} = \sum_{k=1}^n \E\br{(f_{\param}(S) - f_{\param}(S\repk))^2 \,\middle|\, \param, X_1, \ldots, X_k}~.
\]
Then, by Lemma~\ref{lem:canonical}, $(\Delta_{\param^0}, \sqrt{V_{\param^0}})$ is a canonical pair for any fixed $\param^0$.
Hence, $(\Delta_\param, \sqrt{V_\param})_{\param\in \Theta}$ form a canonical family and the conclusions of the previous section hold for it.
\paragraph{Acknowledgements}
We are grateful to Andr\'as Gy\"orgy for many insightful comments.
\bibliographystyle{unsrtnat}
\bibliography{learning}

\appendix
\section{Proofs for semi-empirical concentration inequalities}
\label{sec:proofs_semi_empirical}
\paragraph{Lemma~\ref{lem:canonical} (restated).}
\emph{
  $(\Delta, \sqrt{V})$ is a canonical pair.
  }\\
\begin{proof}
Let $\E_k[\cdot]$ stand for $\E[\cdot\mid X_1,\dots,X_k]$.
The Doob martingale decomposition of $f(S) - \E[f(S)]$ gives
\begin{align*}
f(S) - \E[f(S)] = \sum_{k=1}^n D_k\,,
\end{align*}
where $D_k = \E_k[f(S)] - \E_{k-1}[f(S)] = \E_k[f(S)-f(S\repk)]$
and the last equality follows from the elementary identity $ \E_{k-1}[f(S)] =\E_k[f(S\repk)]$.

Observe that
\[
\Delta = \sum_{k=1}^n D_k \qquad \text{and} \qquad V = \sum_{k=1}^n V_k
\]
and where $V_k = \E_k\br{ \pr{f(S) - f(S\repk)}^2 }$.
Assume for now that for $k\in [n]$, the inequalities
\begin{align}
  \E_{k-1}\br{
    \exp\pr{ \lambda D_k - \frac{\lambda^2}{2} V_k  }
    } 
  \leq 1
  \label{eq:dkbkcp}
\end{align}
hold.
Then, using an argument similar to the proof of McDiarmid's inequality, we get
\begin{align*}
  \E\br{\exp\pr{\lambda A-\frac{\lambda^2}{2} B^2}}
  & =
  \E\br{ 
  	\underbrace{
	\E_{n-1}\br{
    \exp\pr{ \lambda D_n - \frac{\lambda^2}{2} V_n
    }}}_{\le 1}
    \prod_{k=1}^{n-1}
    \exp\pr{ \lambda D_k - \frac{\lambda^2}{2} V_k}
    } 
    \\
  & \le
  \E\br{ 
  	\underbrace{\E_{n-2}\br{
    \exp\pr{ \lambda D_{n-1} - \frac{\lambda^2}{2} V_{n-1}}
    }}_{\le 1}
    \prod_{k=1}^{n-2}
    \exp\pr{ \lambda D_k - \frac{\lambda^2}{2} V_k}
    }
     \\
    & \le \dots \le 1\,.
\end{align*}

Thus, it remains to prove \cref{eq:dkbkcp}.
For this, fix $k\in [n]$ and
introduce a Rademacher variable $\ve \in \spin$ 
such that $\P(\ve = 1) = \P(\ve = -1) = \frac{1}{2}$ and $\ve$ is independent of $S,S'$.
Let $\Delta_k = f(S) - f(S\repk)$.
Using that $\lambda D_k - \frac{\lambda^2}{2} V_k = \E_k[\lambda \Delta_k - \frac{\lambda^2}{2} \Delta_k^2 ]$, we get
\begin{align}
\MoveEqLeft
  \E_{k-1}\br{
    \exp\pr{ \lambda D_k - \frac{\lambda^2}{2} V_k } 
    }
   \leq
  \E_{k-1}\br{
    \exp\pr{ \lambda \Delta_k - \frac{\lambda^2}{2} \Delta_k^2 }
  } \tag{Jensen's w.r.t. $\E_k$}\\
  & =
    \E_{k-1}\br{
    \E_{-k}\E\br{
    \exp\pr{ \ve \lambda  \Delta_k  - \frac{\lambda^2}{2} \pr{\ve  \Delta_k }^2 }
    \,\Big|\, S,S' }
    }
    \,,
    \label{eq:replace_one_symmetrization}
\end{align}
where we recall that the subscript $-k$ in $\E_{-k}[\cdot]$ stands for conditioning on $S$ without $X_k$,
and we get the last equality thanks to our independence assumption, that is
given $X_1,\dots,X_{k-1},X_{k+1},\dots,X_n$,
 $X_k$ and $X_k'$ are identically distributed and hence so are 
 $\Delta_k$ and $-\Delta_k$. 
 Since $x \ve$ is $x^2/2$-subgaussian for $x \in \reals$,
the innermost expectation in~\cref{eq:replace_one_symmetrization} is upper-bounded by one, thus, finishing the proof of \cref{eq:dkbkcp} and also the theorem.
\end{proof}

\paragraph{Lemma~\ref{lem:sgfromexp} (restated).}
\emph{
Let $U$ be a nonnegative valued random variable such that $a =\E\br{\exp(U^2/4)}$ is finite. Then, for any $x\ge 0$,
$\E\br{\exp(x U )} \le a e^{x^2}$ holds.
}
\begin{proof}
Fix $x \geq 0$ and let $\alpha>0$.
Using the inequality $ab \leq (a^2 + b^2)/2$ with $a = x/\sqrt{2 \alpha}$ we have
\begin{align*}
x U = 
\frac{x}{\sqrt{2\alpha}} \sqrt{ 2\alpha U^2} \le \frac{x^2}{4\alpha} + \alpha U^2\,.
\end{align*}
Setting $\alpha = 1/4$, exponentiating both sides and taking expectations the result follows.
\end{proof}

\paragraph{Theorem~\ref{thm:pac_bayes_self_normalized_concentration} (restated).}
\emph{
  For some space $\Theta$, let
$(\Delta_\theta,\sqrt{V_\theta})_{\theta\in \Theta}$ be a canonical family and $S\sim \sD$, jointly distributed with $(\Delta_\theta,\sqrt{V_\theta})_{\theta\in \Theta}$ and taking values in $\sZ$.
Fix an arbitrary probability kernel $\hat{p}$ from $\sZ$ to $\Theta$ and an arbitrary probability measure $p^0$ over $\Theta$, and let $\param \sim \ps$.
Then, for any $x\ge 0$,
  with probability at least $1-2 e^{-x}$ 
  we have that
  \begin{equation}
    \label{eq:pac_bayes_mix_3}
    |\E[\Delta_{\param} \,|\, S]|
    \leq
    \sqrt{2 (\E[V_{\param}] + \E[V_{\param} \,|\, S]) \pr{\KL\pr{\ps \,||\, p^0} + 2 x}}~.
  \end{equation}
  In addition, for all $y > 0$ and $x \geq 2$ with probability at least $1-e^{-x}$ we have
  \begin{equation}
    \label{eq:pac_bayes_mix_4}
    |\E[\Delta_{\param} \,|\, S]|
    \leq
    \sqrt{
      2 \pr{y + \E[V_{\param} \,|\, S]}
      \pr{\KL\pr{\ps \,||\, p^0} + x + \frac{x}{2} \ln\pr{1 + \frac{1}{y} \E[V_{\param} \,|\, S]}}
    }~.
  \end{equation}
  }
\begin{proof}
Applying Chernoff's bounding technique with~\cref{eq:pac_bayes_mix_2} gives
\begin{align*}
  \P\pr{
  \sqrt{\pr{\frac{\E[\Delta_{\param} \,|\, S]^2}{\E[V_{\param}] + \E[V_{\param} \,|\, S]} - 2 \KL\pr{\ps \,||\, p^0}}_+}
  \geq t
  }
  &\leq
  2 \inf_{x \geq 0} e^{x^2 - t x}
  = 2 e^{-\frac{t^2}{4}}~.
\end{align*}
This implies that with probability at least $1-2 e^{-x}$,
\begin{align*}
  \pr{\frac{\E[\Delta_{\param} \,|\, S]^2}{\E[V_{\param}] + \E[V_{\param} \,|\, S]} - 2 \KL\pr{\ps \,||\, p^0}}_+
  \leq
  4 x\,.
\end{align*}
From this, after algebra we get
\begin{align*}
  |\E[\Delta_{\param} \,|\, S]|
  \leq
  \sqrt{(\E[V_{\param}] + \E[V_{\param} \,|\, S]) \pr{2 \KL\pr{\ps \,||\, p^0} + 4 x}}\,,
\end{align*}
showing~\cref{eq:pac_bayes_mix_3}.

Observing that for $t \geq \sqrt{2}$ and $y > 0$,
\begin{align*}
\MoveEqLeft
  \P\pr{ \frac{\E[\Delta_{\param} \,|\, S]^2}{2 (y^2 + \E[V_{\param} \,|\, S])} - \KL\pr{\ps \,||\, p^0} \geq \frac{t^2}{2} \pr{1 + \frac{1}{2} \ln\pr{1 + \frac{\E[V_{\param} \,|\, S]}{y^2}}} }\\
  &\leq
    \P\pr{ \frac{\E[\Delta_{\param} \,|\, S]^2}{2 (y^2 + \E[V_{\param} \,|\, S])} - \KL\pr{\ps \,||\, p^0} \geq \frac{t^2}{2} + \frac{1}{2} \ln\pr{1 + \frac{\E[V_{\param} \,|\, S]}{y^2}} }\\
  &=
    \P\pr{ \frac{\E[\Delta_{\param} \,|\, S]^2}{2 (y^2 + \E[V_{\param} \,|\, S])} - \KL\pr{\ps \,||\, p^0} - \frac{1}{2} \ln\pr{1 + \frac{\E[V_{\param} \,|\, S]}{y^2}} \geq \frac{t^2}{2} }\\
  &\leq
    \E\br{\sqrt{\frac{y^2}{\E[V_{\param} \,|\, S] + y^2}} \, \exp\pr{\frac{\E[\Delta_{\param} \,|\, S]^2}{2 (y^2 + \E[V_{\param} \,|\, S])} - \KL(\ps\,||\,p^0)} } \, e^{-\frac{t^2}{2}}\\
  &\leq
    e^{-\frac{t^2}{2}}~.
\end{align*}
where the last two inequalities follow from Chernoff bound and~\cref{eq:pac_bayes_mix_1}.
This implies that with probability at least $1-e^{-x}$ for all $x \geq 2$,
\begin{align*}
  \frac{
  \frac{\E[\Delta_{\param} \,|\, S]^2}{2 (y^2 + \E[V_{\param} \,|\, S])} - \KL\pr{\ps \,||\, p^0}
  }
  {
  1 + \frac{1}{2} \ln\pr{1 + \frac{\E[V_{\param} \,|\, S]}{y^2}}
  }
  \leq
  x
\end{align*}
and rearranging we get
\begin{align*}
  |\E[\Delta_{\param} \,|\, S]|
  \leq
  \sqrt{
  2 \pr{y^2 + \E[V_{\param} \,|\, S]}
  \pr{\KL\pr{\ps \,||\, p^0} + x \pr{1 + \frac{1}{2} \ln\pr{1 + \frac{1}{y^2} \E[V_{\param} \,|\, S]}}}
  }~.
\end{align*}
This concludes the proof of~\cref{eq:pac_bayes_mix_4}, which is stated with $y$ in place of $y^2$, which does not change the result since $y$ is a free variable.
The proof is now concluded.
\end{proof}

\section{Proofs for applications}
In this sections we discuss some of the implications of our bounds.
\subsection{Bernstein-type Generalization Bounds for Unbounded Losses}
\label{sec:gen_bounds}
g%
Let $\sZ_1 = \dots = \sZ_n$ and
let $\ell \,:\, \Theta \times \sZ_1 \to [0, \infty)$ be some loss function.
Recall that the population loss and the empirical loss is defined as
\[
  L(\param) = \E[\ell(\param, X_1')] \qquad \text{and} \qquad
  \wh{L}_S(\param) = \frac1n \sum_{k=1}^n \ell(\param, X_k)~,
\]
respectively.
Then, \cref{thm:pac_bayes_self_normalized_concentration} with
\[
  V_{\param} = \sum_{k=1}^n \E\br{\pr{\wh{L}_S(\param) - \wh{L}_{\Srepk}(\param)}^2 \,\middle|\, \param, X_1, \ldots, X_k}
\]
implies the following semi-empirical PAC-Bayesian generalization bound:
\begin{cor}
  \label{cor:gen_bound}
  Assume that the elements of $S=(X_1,\dots,X_n)\in \sZ$ are sampled independently from each other.
  Let $\hat{p}$ be a probability kernel from $\sZ$ to $\Theta$ and let $p^0 \in \sM_1(\Theta)$ be a probability distribution over $\Theta$.
  Then, for any $x\ge 2$, with probability at least $1-e^{-x}$, we have
  \begin{align*}
    \abs{\E\br{ \wh{L}_S(\param) - L(\param) \,\middle|\, S} }
    \leq
    \sqrt{
    2 \pr{y + \E[V_{\param} \,|\, S]}
    \pr{\KL\pr{\ps \,||\, p^0} + x + \frac{x}{2} \ln\pr{1 + \frac{1}{y} \E[V_{\param} \,|\, S]}}
    }
  \end{align*}
where
\begin{align}
  \label{eq:L_V_param}
  \E[V_{\param} \,|\, S]
  \leq
  \frac{1}{n^2}
  \E\br{\sum_{k=1}^n \ell(\param, X_k)^2 + \ell(\param, X_k')^2 \,\middle|\, S}~.
\end{align}
\end{cor}
\begin{proof}
  According to notation of \cref{sec:paces} we choose $f_{\param}(z) = \wh{L}_z(\param)$, and note that
  \begin{align*}
    \E[\Delta_{\theta} \,|\, S]
    &=
      \E\br{ \wh{L}_S(\theta) - \int \wh{L}_s(\param) \, \sD(\diff s) \,\middle|\, S }\\
    &=
      \E\br{ \wh{L}_S(\theta) \,|\, S } - \E\br{L(\theta) \,|\, S }\,
  \end{align*}
  and that $(\Delta_{\param}, \sqrt{V_{\param}})_{\param \in \Theta}$ form a canonical family as described in \cref{sec:paces}.
  Thus, we can apply \cref{eq:pac_bayes_mix_4} of \cref{thm:pac_bayes_self_normalized_concentration}.
At the same time,
$\wh{L}_S(\param) - \wh{L}_{\Srepk}(\param)=\frac1n \pr{\ell(\theta,X_k)-\ell(\theta,X_k')}$, and so
\begin{align*}
  V_{\param}
  &=
    \sum_{k=1}^n \E\br{\pr{\wh{L}_S(\param) - \wh{L}_{\Srepk}(\param)}^2 \,\middle|\, \param, X_1, \ldots, X_k}\\
  &=
    \frac{1}{n^2} \sum_{k=1}^n \E\br{\pr{\ell(\param, X_k) - \ell(\param, X_k')}^2 \,\middle|\, \param, X_1, \ldots, X_k}\\
  &\leq
    \frac{1}{n^2} \pr{ \sum_{k=1}^n \ell(\param, X_k)^2 + \sum_{k=1}^n \E\br{\ell(\param, X_k')^2 \,|\, \param} }~,
\end{align*}
where the inequality used that $X_k'$ is independent of $\theta$ and $S$.
Taking expectations of both sides, conditioned on $S$,
and plugging into \cref{eq:pac_bayes_mix_4} completes the proof.
\end{proof}
Note that one simple data-independent choice for $y$ is $y = 1/n^2$ which gives a bound
\[
  \abs{\E\br{ \wh{L}_S(\param) - L(\param) \,\middle|\, S} }
  \leqCln \frac1n
  \pr{1 + \sqrt{\E\br{\sum_{k=1}^n \ell(\param, X_k)^2 + \ell(\param, X_k')^2 \,\middle|\, S} \KL(\ps\,||\,p^0)}}~.
\]
Thus we pay a logarithmic term for an agnostic choice of $y$.
Of course, when we have some idea about the range of $\E[V_{\param} \,|\, S]$
(or when the loss function is bounded), we can follow a more refined argument and tune $y$ optimizing the bound over a quantized range for $y$.
Then, the final bound can be obtained by taking a union bound.
\subsection{Bounded Losses}
In this section we consider a simple case when the loss function is bounded, i.e.\ $\ell \,:\, \Theta \times \sZ_1 \to [0, 1]$, and apply \cref{cor:gen_bound} with the choice $y = 1/n^2$.
In this case we will have a simple variance proxy $\wh{\sigma}^2_{\param} = \frac1n \sum_{k=1}^n \ell(\param, X_k)^2$ controlling the generalization gap.
This is enough to get fully empirical bounds:

\begin{theorem}
  \label{thm:bounded_losses}
  For any $x \geq 2$, with probability at least $1 - 2 e^{-x}$ we have
  \begin{align*}
    \abs{\E\br{ \wh{L}_S(\param) - L(\param) \,\middle|\, S} }
    \leq
    \sqrt{2 \pr{\frac{1}{n^2} + \frac2n U_S} C_x(S)}
  \end{align*}
  where
  \begin{align*}
    C_x(S) = \KL\pr{\ps \,||\, p^0} + x + x \ln\pr{\sqrt{1 + n}}
  \end{align*}
  and
  \[
    U_S = \E\br{\wh{\sigma}^2_{\param} \,\middle|\, S} + \sqrt{\frac{2}{n} \E\br{\wh{\sigma}^2_{\param} \,\middle|\, S} C_x(S)} + \frac1n \pr{2C_x(S) + \sqrt[4]{2}\sqrt{C_x(S)} + \frac{1}{\sqrt{2}}}~.
  \]
\end{theorem}
The theorem implies that with high probability,
\[
  \abs{\E\br{ \wh{L}_S(\param) - L(\param) \,\middle|\, S} }
  \leqCln
  \sqrt{\frac{\E[\wh{\sigma}^2_{\param}\mid S]\, \KL\pr{\ps \,||\, p^0}}{n}} + \frac{\KL\pr{\ps \,||\, p^0}}{n} + \frac1n~,
\]
which is not hard to see noting that $\sqrt{C_x(S)} \leq C_x(S) \leqCln \KL\pr{\ps \,||\, p^0}$ (since $x \geq 2$)
and by completing the square $U_S/n \leqCln \pr{\sqrt{\E[\wh{\sigma}^2_{\param}\mid S]/n} + \sqrt{C_x(S)}/n}^2$.

Observe that the bound becomes of order $\KL\pr{\ps \,||\, p^0}/n$ for the small enough variance proxy term $\E[\wh{\sigma}^2_{\param} \mid S]$.
\paragraph{Proof of \cref{thm:bounded_losses}.}
Fix $x\ge 2$.
Noting that $\E[V_{\param} \,|\, S] \leq 1/n$ a.s. by boundedness of the loss and taking $y = 1/n^2$, \cref{cor:gen_bound} implies that with probability at least $1-e^{-x}$,
\begin{align*}
  \abs{\E\br{ \wh{L}_S(\param) - L(\param) \,\middle|\, S} }
  \leq
  \sqrt{2 \pr{\frac{1}{n^2} + \E[V_{\param} \,|\, S]} C_x(S)}~.
\end{align*}
All that is left is to give an upper bound $\E[V_{\param} \,|\, S]$ in terms of empirical quantities.
Introduce notation
\[
  \sigma^2_{\param}(s) = \frac1n \sum_{k=1}^n \ell(\param, z_k)^2~, \qquad s=(z_1,\ldots,z_n) \in \sZ~,
\]
and so $\wh{\sigma}^2_{\param} = \sigma^2_{\param}(S)$.
Thus, by~\cref{eq:L_V_param},
\begin{equation}
  \label{eq:V_self_bounding_1}
  \E[V_{\param} \,|\, S] \leq \frac1n \E\br{\sigma^2_{\param}(S) \,\middle|\, S} + \frac1n \E\br{\sigma^2_{\param}(S') \,\middle|\, S}
\end{equation}
and we only need to upper bound $\E\br{\sigma^2_{\param}(S') \,\middle|\, S}$, which we will do apply our PAC-Bayesian concentration inequality.
In particular we apply \cref{thm:pac_bayes_self_normalized_concentration} to $\wh{\sigma}^2_{\param}$.
Taking
$f_{\param}(s) = \sigma^2_{\param}(s)$ and denoting a deviation by
\[
  \Delta_{\param}' = \sigma_{\param}^2(S) - \int \sigma_{\param}^2(s) \sD(\diff s)
\]
we observe that
$\E[\Delta_{\param}' \,|\, S] = \E[\sigma_{\param}^2(S) \,|\, S] - \E[\sigma_{\param}^2(S') \,|\, S]$
and that an \ac{ES} variance proxy of $\wh{\sigma}^2$ is 
\begin{align*}
  V_{\param}'
  &=
    \frac{1}{n^2} \sum_{k=1}^n \pr{\ell(\param, X_k)^2 - \ell(\param, X_k')^2}^2\\
  &\leq
    \frac{2}{n^2} \sum_{k=1}^n \pr{\ell(\param, X_k) - \ell(\param, X_k')}^2 = 2 V_{\param}
\end{align*}
by boundedness of the loss.
By Lemma~\ref{lem:canonical}, $(\Delta_{\param}', \sqrt{V_{\param}'})_{\param\in \Theta}$ forms a canonical family, and so
by \cref{eq:pac_bayes_mix_4} of \cref{thm:pac_bayes_self_normalized_concentration}
with $y=1/n^2$ and having that $\E[V_{\param} \,|\, S] \leq 1/n$, this gives that for all $x\ge 2$, with probability at least $1-e^{-x}$,
\begin{align*}
  \E[\sigma^2_{\param}(S') \,|\, S] - \E\br{\sigma^2_{\param}(S) \,\middle|\, S}
  &\leq
    \sqrt{2 \pr{\frac{1}{n^2} + \E[V_{\param}' \,|\, S]} \pr{\KL\pr{\ps \,||\, p^0} + x + \frac{x}{2} \ln\pr{1 + n}}}\\
  &\leq
    \sqrt{2 \pr{\frac{1}{n^2} + 2 \E[V_{\param} \,|\, S]} C_x(S)}~.
\end{align*}
This combined with~\cref{eq:V_self_bounding_1} and using subadditivity of $\sqrt{\cdot}$ gives
\[
  \E[V_{\param} \,|\, S]
  \leq
  \frac2n \E\br{\sigma^2_{\param}(S) \,\middle|\, S} + \frac{\sqrt{2}}{n^2} + \frac2n \sqrt{\E\br{V_{\param} \,|\, S} C_x(S)}~.
\]
Using the fact that for $a,b,c \geq 0$, $a \leq b + c \sqrt{a} \ \Rightarrow \ a \leq b + c^2 + \sqrt{b} c$, we get
\begin{align*}
  \E[V_{\param} \,|\, S]
  &\leq
    \frac2n \E\br{\sigma^2_{\param}(S) \,\middle|\, S} + \frac{\sqrt{2}}{n^2}
    + \frac{4}{n^2} C_x(S)\\
    &+ \sqrt{\frac2n \E\br{\sigma^2_{\param}(S) \,\middle|\, S} + \frac{\sqrt{2}}{n^2}} \pr{\frac2n \sqrt{C_x(S)}}\\
  &\leq
    \frac2n \E\br{\sigma^2_{\param}(S) \,\middle|\, S}
    + \frac{2}{n^2} \pr{\frac{1}{\sqrt{2}} + 2 C_x(S) + \sqrt[4]{2}\sqrt{C_x(S)}}\\
    &+ 2 \sqrt{\frac{2}{n^3} \E\br{\sigma^2_{\param}(S) \,\middle|\, S} C_x(S)}~.
\end{align*}
Simplifying and taking a union bound completes the proof.

\subsection{Concentration of Weighted Averages}
\label{sec:wa}
In this section we cover concentration and PAC-Bayesian results on \emph{weighted averages} (also known as self-normalized estimators).
Our results are particularly handy for \emph{unbounded} weights since bounds we presented in previous sections hold for functions of independent, but not necessarily bounded random variables.
Formally, let $X_i = (W_i,R_i)\in [0,\infty)\times [0,1]$, $i \in [n]$ be a sequence of random pairs (not necessarily independent from each other)
and let $S= \pr{X_1,\dots,X_n}$ and $S'=\pr{X_1', \dots, X_n'}$ be sampled as before.
Throughout this section we look at the weighted averages of the form
\begin{equation}
  \label{eq:wm}  
  \fwa(S) = \frac{\sum_{i=1}^n W_i R_i}{\sum_{j=1}^n W_j}\,.  
\end{equation}
We call $W_i$ the weight of the $i$th data $R_i$.
The first goal we pursue here is to derive high probability tail bounds on $\fwa(S) - \E[\fwa(S)]$.
In these bounds we may assume that the distribution of $W_i$ is known.
Note that weights are unbounded from above.
This is important, as in some applications the support of the distribution of $W_i$ is indeed unbounded, while in other applications where $W_i$ is bounded, tail inequalities that use almost sure upper bounds on $W_i$ become very lose.
For instance, in a \emph{\acl{WIS}} estimator, $W_i$ is a ratio of two probability densities, and therefore it is often indeed unbounded from above.

As a corollary to \cref{thm:self_norm_concentration}, we obtain the following result:
\begin{theorem}
  \label{cor:wm}
For any $y > 0$ and any $\conf \geq 2$, with probability at least $1-e^{-\conf}$ we have
\[
  \abs{\fwa(S) - \E[\fwa(S)]} \leq \sqrt{2 (2 \Vwa + y) \pr{1 + \ln\pr{\sqrt{1 + 2 \Vwa / y}}} x}
\]
where
\begin{equation}
  \label{eq:semi_empirical_V}
  \Vwa = \sum_{k=1}^n \E\br{\tilde W_k^2 + \tilde U_k^2 \,\middle|\, W_1, \ldots, W_k }
\end{equation}
and
\[
  \tilde W_k = \frac{W_k}{\sum_{j=1}^n W_j} \qquad \tilde U_k = \frac{W_k'}{W_k' + \sum_{j\ne k}W_j}
\]
for $k\in [n]$,
and $W_k'$ shares the distribution of $W_k$ and is independent of $(W_1,\dots,W_n)$.
\end{theorem}
Note that $\tilde W_k$, $\tilde U_k$ are nonnegative and sum to one.
As before, $y$ can be chosen from some feasible range (through the union bound), for instance a
quantized range $[0,2]$ (since $\Vwa \leq 2$).
Further, if the distribution of $(W_i)_i$ available, $\Vwa$ can be computed exactly.
However, sometimes this can be computationally prohibitive.
In the following we propose two computationally-amenable bounds on $\Vwa$, which exploit the fact that weights concentrate well in the vicinity of $0$.
The following lemma (with proof given in the appendix) captures this fact.
\begin{lemma}
  \label{lem:bernstein_lower_tail}
  Assume that non-negative random variables $W_1, W_2, \ldots, W_n$ are distributed i.i.d.
  Then, for any $t \in [0, n \E[W_1])$,
  \begin{align*}
    \P\pr{\sum_{i=1}^n W_i \leq t} \leq \exp\pr{- \frac{\pr{t - n \E\br{W_1}}^2}{2 n \E\br{W_1^2}}}~.
  \end{align*}
  Also, with probability at least $1-e^{-\conf}$ for $\conf > 0$,
  \begin{equation}
    \label{eq:b_n}
    \sum_{i=1}^n W_i \geq N_x(n)
    \quad \text{where} \quad
    N_x(n) = \pr{n \E[W_1] - \sqrt{2 \conf n \E\br{W_1^2}}}_+~.
  \end{equation}
\end{lemma}
This lemma implies the following bounds on $\Vwa$.
\begin{prop}
  \label{prop:V_bounds}
  Let $\Vwa$ be defined as in~\eqref{eq:semi_empirical_V} and suppose that weights $W_1, W_2, \ldots, W_n$ are distributed i.i.d.
  Then for any $\conf > 0$ with probability at least $1-n e^{-\conf}$ we have
  \begin{equation}
    \label{eq:V_semi_empirical_bound}
    \Vwa \leq
    \sum_{k=1}^n\pr{
    \frac{W_k^2}{\pr{\sum_{i=1}^k W_i + N_x(n-k)}^2}
    +
    \frac{\E[{W_k'}^2]}{\pr{\sum_{i=1}^{k-1} W_i + N_x(n-k+1)}^2}
    }~.
  \end{equation}
  Under the same conditions we also have that
  \begin{equation}
    \label{eq:V_concentration_bound}
    \Vwa \leq \frac{1}{N_x(n)^2} \sum_{k=1}^n \pr{W_k^2 + \E[{W_k'}^2]}~.
  \end{equation}
\end{prop}
\begin{proof}
  Observe that
\begin{align*}
  \E_k[\tilde W_k^2]
  = \E_k\br{\frac{W_k^2}{\pr{\sum_{i=1}^k W_i + \sum_{i=k+1}^n W_i}^2}}
  \leq \frac{W_k^2}{\pr{\sum_{i=1}^k W_i + N_x(n-k)}^2}
\end{align*}
where the last inequalities holds w.p.\ at least $1-n e^{-\conf}$ for $\conf > 0$ by taking union bound.
Similarly
\begin{align*}
  \E_k[\tilde U_k^2]
  = \E_k\br{ \frac{{W_k'}^2}{\pr{\sum_{i=1}^{k-1} W_i + W_k' + \sum_{i=k+2}^n W_i}^2} }
  \leq \frac{\E[{W_k'}^2]}{\pr{\sum_{i=1}^{k-1} W_i + N_x(n-k+1)}^2}~.
\end{align*}
Obviously, the last result follows by concentrating the entire sum in the denominator.
\end{proof}
Combined with Theorem~\ref{cor:wm} through the union bound, these yield computationally efficient concentration bounds for $\fwa$.

\paragraph{Proof of Theorem~\ref{cor:wm}.}
  The proof boils down to application of Theorem~\ref{thm:stability_vrep}.
  First we show a basic remove-one stability property of a self-normalized average.
  Let $\fwa_k(S\delk) = \frac{\sum_{i\ne k} W_i R_i}{\sum_{i\ne k} W_i}$ be the remove-one version of $\fwa(S)$, where $S\delk = (X_1,\dots,X_{k-1},X_{k+1},\dots,X_n)$.
\begin{prop}[Remove-One Stability]
  \label{prop:remove_stab}
  Let $E_k = R_k - \fwa_k(S\delk)$ denote a pointwise remove-one (or leave-one-out) error.
  Then, with $f$ defined by~\eqref{eq:wm}, for any $k\in [n]$,
  \[
    \fwa(S) - \fwa_k(S\delk) = \frac{W_k E_k}{\sum_{j=1}^n W_j}\,.
  \]
\end{prop}
\begin{proof}
The statement follows from simple algebra.
%
\end{proof}
As before, let $\E_k[\cdot]$ stand for $\E[\cdot\mid X_1,\dots,X_k]$.
We need to upper bound 
\begin{align*}
    \Vtilrep = \sum_{k=1}^n \E_k\br{(\fwa(S) - \fwa(S\repk))^2}
\end{align*}
and its expectation.
We have $\fwa(S) - \fwa(S\repk) = \fwa(S) - \fwa_k(S\delk) + \fwa_k(S\delk) - \fwa(S\repk)$. Squaring both sides and using $(a+b)^2 \le 2(a^2 + b^2)$, we see that we need to bound $\E_k[(\fwa(S) - \fwa_k(S\delk))^2]$ and $\E_k[(\fwa(S\repk) - \fwa_k(S\delk))^2]$.
Recall that $\tilde W_k = W_k/\sum_j W_j$ is the $k$-th normalized weight.
We can directly use \cref{prop:remove_stab} to bound the first of these terms by $\E_k[\tilde W_k^2]$ (using that $E_k^2\le 1$).
The second term can be bound the same way, except now we start from $\fwa(S\repk) - \fwa_k(S\delk) = W_k' E_k'/(W_k'+\sum_{j\ne k} W_j)$, where $E_k' = R_k' - \fwa_k(S\delk)$.
Recall that $\tilde U_k = W_k'/(W_k' + \sum_{j\ne k}W_j)$.
Putting things together, we have
\begin{align*}
\Vtilrep \leq 2 \Vwa = 2 \sum_{k=1}^n \pr{ \E_k[\tilde W_k^2 ] + \E_k[ \tilde U_k^2 ] }\,.
\end{align*}
Applying the first result of Theorem~\ref{thm:stability_vrep} completes the proof.

\subsubsection{Off-policy Evaluation through Weighted Importance Sampling}
\label{sec:offpoleval}
Recall that in the setting of off-policy evaluation we assume that action-reward pairs are distributed according to some joint probability measure $\sD \in \sM_1([K] \times [0,1])$, and observations $(A_1, R_1), \dots, (A_n, R_n)$ are generated by sampling actions $A_i \sim \pi_b$, where $\pi_b \in \sM_1([K])$ is called the \emph{behavior} policy and $R_i \sim \sD(R \,|\, A_i)$.
Now given another distribution $\pi \in \sM_1([K])$ called the \emph{target policy}, we want to estimate its expected reward, or the \emph{value} function
\[
  v(\pi) = \sum_{a \in [K]} \pi(a) \E[R | A=a]~.
\]
This can be done by employing \ac{WIS} estimator, which is a special case of the weighted average $\fwa$ with weights $W_i = \pi(A_i)/\pi_b(A_i)$.
Since in this case, the weights $W_i$ are not necessarily bounded, the tools we have developed in previous section allows to do exactly that.
In the following, for the choice of such weights we denote $\fwa(S)$ by $\wh{v}\WIS$ and its variance proxy $\Vwa$ by $V\WIS$.
\begin{theorem}
  \label{thm:off_policy_eval}
  For any $y > 0$ and $x \geq 2$, with probability at least $1-(n+1) e^{-x}$ we have
    \begin{align*}
    v(\pi)
    \geq
    \frac{N_x(n)}{n}
    \pr{\wh{v}\WIS
    - \sqrt{2 (2 V\WIS + y) \pr{1 + \ln\pr{\sqrt{1 + 2 V\WIS / y}}} x}}_+~.
  \end{align*}
  where $V\WIS$ is defined as in~\eqref{eq:semi_empirical_V} and $N_x(n)$ is defined as in~\eqref{eq:b_n}.
\end{theorem}
\begin{proof}  
  Introduce decomposition
  \begin{align*}
    v(\pi) - \wh{v}\WIS = \underbrace{v(\pi) - \E[\wh{v}\WIS]}_{\text{Bias of \ac{WIS} estimator}}
    + \underbrace{\E[\wh{v}\WIS] - \wh{v}\WIS}_{\text{Concentration of \ac{WIS} estimator}}~.
  \end{align*}
  Observe that concentration is readibly given by the Theorem~\ref{cor:wm}, thus we pay attention to the bias.

  We apply Lemma~\ref{lem:bernstein_lower_tail} getting lower bound on the sum of weights (note that weights are independent from each other) to get
  \begin{align*}
    \E[\wh{v}\WIS]
    &= \E\br{ \frac{\sum_{i=1}^n W_i R_i}{\sum_{i=1}^n W_i} }
      \leq \frac{1}{N_x(n)} \E\br{\sum_{i=1}^n W_i R_i}\\
    &= \frac{n}{N_x(n)} \int_0^1 \sum_{a \in [K]} \frac{\pi(a)}{\pi_b(a)} \, r \, \pi_b(a) \sD(r \,|\, a)
    = \frac{n}{N_x(n)} \, v(\pi)~.
  \end{align*}
  Thus, bias is bounded as
  \[
    v(\pi) - \E[\wh{v}\WIS] \geq v(\pi) \pr{1 - \frac{n}{N_x(n)}}~.
  \]
  Combining this with the concentration result of Theorem~\ref{cor:wm} according to the decomposition, with probability at least $1-2 e^{-x}$ for $x \geq 0$ gives  
  \begin{align*}
    v(\pi)
    \geq
    \wh{v}\WIS
    + v(\pi) \pr{1 - \frac{n}{N_x(n)}}
    - C_x(S)
  \end{align*}
  where $C_x(S) = \sqrt{2 (2 V\WIS + y) \pr{1 + \ln\pr{\sqrt{1 + 2 V\WIS / y}}} x}$
  and rearranging we get a desired result
  \begin{align*}
    v(\pi)
    \geq
    \frac{N_x(n)}{n}
    \pr{\wh{v}\WIS
    - C_x(S)}
  \end{align*}
  which completes the proof.
\end{proof}
Note that $V\WIS$ can be bounded according to Proposition~\ref{prop:V_bounds}.

\subsection{PAC-Bayesian Bound for \acl{WIS}}
\label{sec:pac_bayesian_wis}
In this section we specialize the framework discussed in \cref{sec:paces} and consider a class of parametric target policies $\cbr{\pi_{\vartheta} \in \sM_1([K]) \,:\, \vartheta \in \Theta}$ where $\Theta$ is a parameter space, and as in~\cref{sec:pacbayes}, parameter $\param \sim \ps \in \sM_1(\Theta)$.
Note that density $\ps$ depends on the tuple of observed action and rewards $S = \pr{(A_1, R_1), \ldots, (A_n, R_n)}$ generated as described in Section~\ref{sec:offpoleval} ($S'$ is still sampled independently from $S$).
The importance weights are now defined w.r.t.\ the random parameter $\param$, that is $W_{\param,i} = \pi_{\theta}(A_i) / \pi_b(A_i)$, $i \in [n]$.
Note that unlike the case of evaluation, importance weights $W_{\param,i}$ are not independent from each other anymore, because $\param$ is a function of $S$
(the independence, however, still holds for weights $W'_{\param,i}$).
However, $\param$-dependent \ac{WIS} estimator $\wh{v}\WIS_{\param}$ is a function of $S$ (action-reward pairs independent from each other), so all our concentration arguments still apply.
In the following $N_{\param,x}$ stands for its counterpart $N_x$ with weights depending on $\param$.
\begin{theorem}
  \label{thm:offpollearn}
  Fix an arbitrary probability kernel $\hat{p}$ from $\sZ$ to $\Theta$ and an arbitrary probability measure $p^0$ over $\Theta$.
  Then for any $y > 0$ and any $n \geq 1, x \geq 2$ that satisfy $N_x(n) > 0$, with probability at least $1-2 e^{-x}$,
\begin{align*}
  \E[v(\pi_{\param}) \,|\, S]
  \geq
  \pr{
  \E[\wh{v}\WIS_{\param} \,|\, S]
  - \min\cbr{1, \E\br{ \abs{\frac{n}{N_{\param, x}(n)} - 1} \,\middle|\, S}}
  - \sqrt{
  2 \pr{y + \E[V\WIS_{\param} \,|\, S]}
    C_{x,y}(S)
  }
  }_+
\end{align*}
where
$C_{x,y}(S) = \KL\pr{\ps \,||\, p^0} + x + x \ln\pr{\sqrt{1 + \E[V\WIS_{\param} \,|\, S] / y}}$
and
\begin{equation}
  \label{eq:semi_empirical_V_param}
  V\WIS_{\param} = \sum_{k=1}^n \E\br{\tilde W_{\param,k}^2 + \tilde U_{\param,k}^2 \,\middle|\, \param, A_1, \ldots, A_k }
\end{equation}
and
\[
  \tilde W_{\param,k} = \frac{W_{\param,k}}{\sum_{j=1}^n W_{\param,j}} \qquad \tilde U_{\param,k} = \frac{W_{\param,k}'}{W_{\param,k}' + \sum_{j\ne k}W_{\param,j}}~.
\]
\end{theorem}
\begin{proof}
Throughout the proof we use notation $\wh{v}\WIS_{\param}(S)$ to indicate that the estimator is evaluated on $S$.
We start from decomposition
\begin{align*}
  \E[v(\pi_{\param}) \,|\, S] - \E[\wh{v}\WIS_{\param}(S) \,|\, S]
  = \underbrace{\E[v(\pi_{\param}) \,|\, S] - \E[\wh{v}\WIS_{\param}(S') \,|\, S]}_{\text{Bias}}
  + \underbrace{\E[\wh{v}\WIS_{\param}(S') \,|\, S] - \E[\wh{v}\WIS_{\param}(S) \,|\, S]}_{\text{Concentration}}
\end{align*}
and first handle the bias term.
Since elements of $W_{\param,i}'$, $i \in [n]$ are independent from each other, \cref{eq:b_n} gives us that
\[
  \sum_{i=1}^n W_{\param, i}' \geq N_{\param,x}(n) \qquad \text{with probability at least } 1-e^{-x} \text{ for any } x \geq 0
\]
and thus
\begin{align*}
  \E[\wh{v}\WIS_{\param}(S') \,|\, S]
  &=
    \E\br{ \frac{\sum_{i=1}^n W_{\param, i}' R_i'}{\sum_{i=1}^n W_{\param, i}'} \,\middle|\, S}\\
  &\leq
    \E\br{ \frac{1}{N_{\param, x}(n)} \sum_{i=1}^n W_{\param, i}' R_i' \,\middle|\, S}\\
  &=
    \E\br{ \frac{n}{N_{\param, x}(n)} \E\br{W_{\param, 1}' R_1' \,|\, \param} \,\middle|\, S} \tag{$(A'_i, R'_i) \ i \in [n]$ distributed identically}\\
  &=
    \E\br{ \frac{n}{N_{\param, x}(n)} v(\pi_{\param}) \,\middle|\, S}~.
\end{align*}
Observing that rewards are bounded by $1$, we have a minimum of two bounds
\begin{align*}
  \E[\wh{v}\WIS_{\param}(S') \,|\, S] \leq \min\cbr{1, \E\br{ \frac{n}{N_{\param, x}(n)} v(\pi_{\param}) \,\middle|\, S}}~.
\end{align*}
This shows that the bias is bounded as
\begin{align*}
  \E[\wh{v}\WIS_{\param}(S') \,|\, S] - \E[v(\pi_{\param}) \,|\, S]
  &\leq
    \min\cbr{1 - \E[v(\pi_{\param}) \,|\, S], \E\br{ \pr{\frac{n}{N_{\param, x}(n)} - 1} v(\pi_{\param}) \,\middle|\, S}}\\
  &\leq
    \min\cbr{1, \E\br{ \abs{\frac{n}{N_{\param, x}(n)} - 1} \,\middle|\, S}}
\end{align*}
The concentration term then follows from~\cref{thm:pac_bayes_self_normalized_concentration} where $V_{\param}$ is a semi-empirical Efron-Stein variance proxy, that is
\begin{align*}
  \E[\wh{v}\WIS_{\param}(S) \,|\, S] - \E[\wh{v}\WIS_{\param}(S') \,|\, S]
  \leq
  \sqrt{
  2 \pr{y + \E[V\WIS_{\param} \,|\, S]}
  C_{x,y}(S)
  }~.
\end{align*}
This completes the proof.
\end{proof}

\section{Other Proofs}
\paragraph{Proof of Lemma~\ref{lem:bernstein_lower_tail} (see also~\citep{maurer2003bound}).}
Chernoff bound readily gives a bound on the lower tail
  \begin{align*}
    \P\pr{\sum_{i=1}^n X_i \leq t} \leq \inf_{\lambda > 0} e^{\lambda t} \E\br{e^{- \lambda \sum_{i=1}^n X_i}}~.
  \end{align*}
  By independence of $X_i$
  \begin{align*}
    \prod_{i=1}^n \E\br{e^{- \lambda X_i}}
    &\leq
    \prod_{i=1}^n \pr{1 - \lambda \E\br{X_i} + \frac{\lambda^2}{2} \E\br{X_i^2} } \tag{$e^{-x} \leq 1 - x + \frac{1}{2} x^2$ for $x \geq 0$}\\
    &\leq
    e^{- \lambda n \E\br{X_1} + \frac{\lambda^2 n}{2} \E\br{X_1^2}} \tag{$1+x \leq e^x$ for $x \in \reals$ and i.i.d.\ assumption}
  \end{align*}
  %
  Getting back to the Chernoff bound gives,
  \begin{align*}
    \lambda = \max\cbr{\frac{n \E\br{X_1} - t}{n \E\br{X_1^2}}, 0}~.
  \end{align*}  
  This proves the first result.
  The second result comes by inverting the bound and solving a quadratic equation.

\end{document}